\documentclass[final]{cvpr}

\usepackage{times}
\usepackage{epsfig}
\usepackage{graphicx}
\usepackage{amsmath}
\usepackage{amssymb}

\usepackage[pagebackref=true,breaklinks=true,colorlinks,bookmarks=false]{hyperref}
\usepackage{array}

\def\cvprPaperID{****} 
\def\confYear{CVPR 2021}

\usepackage{graphicx}
\usepackage{amsmath}
\usepackage{algorithm}
\usepackage[noend]{algpseudocode}
\usepackage{wrapfig}
\usepackage{amsthm}
\usepackage{amssymb}
\usepackage{color}
\usepackage{xcolor}
\usepackage{array}
\usepackage{mathabx}

\newtheorem{theorem}{Theorem}

\newtheorem{assumption}{Assumption}
\newtheorem{remark}{Remark}

\theoremstyle{definition}

\newcommand{\X}{\mathcal{X}}

\definecolor{light-gray}{gray}{0.95}

\begin{document}

\title{MeanShift++: Extremely Fast Mode-Seeking With Applications to Segmentation and Object Tracking}

\author{Jennifer Jang\\
Waymo\\
{\tt\small jangj@waymo.com}
\and
Heinrich Jiang\\
Google Research\\
{\tt\small heinrichj@google.com}
}

\maketitle

\begin{abstract}
   MeanShift is a popular mode-seeking clustering algorithm used in a wide range of applications in machine learning. However, it is known to be prohibitively slow, with quadratic runtime per iteration. We propose MeanShift++, an extremely fast mode-seeking algorithm based on MeanShift that uses a grid-based approach to speed up the mean shift step, replacing the computationally expensive neighbors search with a density-weighted mean of adjacent grid cells. In addition, we show that this grid-based technique for density estimation comes with theoretical guarantees. The runtime is linear in the number of points and exponential in dimension, which makes MeanShift++ ideal on low-dimensional applications such as image segmentation and object tracking. We provide extensive experimental analysis showing that MeanShift++ can be more than 10,000x faster than MeanShift with competitive clustering results on benchmark datasets and nearly identical image segmentations as MeanShift. Finally, we show promising results for object tracking.
\end{abstract}

\section{Introduction}

MeanShift \cite{cheng1995mean,comaniciu1999mean,fukunaga1975estimation} is a classical mode-seeking clustering algorithm that has a wide range of applications across machine learning and computer vision. Recent applications within computer vision include object tracking \cite{ning2012robust,vojir2014robust,leichter2010mean}, unsupervised image segmentation \cite{tao2007color,carreira2006acceleration,zhou2013mean}, video segmentation \cite{paris2007topological,dementhon2002spatio,paris2008edge}, image restoration \cite{arjomand2017deep,bigdeli2017image}, edge-preserving smoothing \cite{paris2008edge,barash2004common,buschenfeld2012edge}, point clouds \cite{lee2009shoreline,vosselman2013point,yue2018new}, and remote sensing \cite{ming2012semivariogram,ponti2012segmentation,chen2018airborne,michel2014stable}. More broadly in machine learning, MeanShift has been used for semi-supervised clustering \cite{anand2013semi,tuzel2009kernel}, manifold denoising \cite{xiang2016statistical,wang2010manifold}, matrix completion \cite{wang2011denoising,choudhary2016active}, anomaly detection \cite{aydin2013robust,yu2013scalable,tsai2010mean}, as well as numerous problems in medical imaging \cite{bai2013novel,tek2001vessel,zhou2011gradient,mayer2009adaptive,tek2010method,nguyen2012clustering,zhou2009anisotropic,zhou2014semi}, wireless sensor networks \cite{yu2020mean,zhou2009mean,xie2014k,sapre2018moth,wu2011video,qu2018incorporating}, and robotics \cite{kroemer2009active,lakaemper2009simultaneous,hu2013design,kato2005optimizing,yang2014robotic,cha2011mahru}.

Given a set of examples, MeanShift proceeds in iterations, where in each iteration, each point is moved to the average of the points within a neighborhood ball centered at that point. The radius of the ball is a hyperparameter, often referred to as the {\it bandwidth} or {\it window size}. All initial examples that converge to the same point are clustered together and the points of convergence are estimates of the modes or local maximas of the probability density function. It has been shown that MeanShift implicitly performs a gradient ascent on the kernel density estimate of the examples \cite{arias2016estimation}. MeanShift thus serves two purposes: mode-seeking and clustering.

MeanShift is often an attractive choice because it is non-parametric: unlike popular objective-based clustering algorithms such as $k$-means \cite{arthur2006k,kanungo2002efficient} and spectral clustering \cite{ng2002spectral,von2007tutorial}, it does not need to make many assumptions on the data, and the number of clusters is found automatically by the algorithm rather than a hyperparameter that needs to be set. In other words, MeanShift can adapt to general probability distributions. However, one of the main drawbacks of this procedure is its computational complexity: each iteration requires $O(n^2)$ computations. This is because for each example, calculating the window around the example is linear time in the worst case.

In this paper, we propose MeanShift++, a simple but effective procedure which first partitions the input space into a grid. Then, at each iteration, each point is assigned to its appropriate grid cell. We then approximate any point's window with the average point in its and its neighboring grid cells. Each iteration in this procedure runs in linear time to the number of data points, with the cost of being exponential to the dimension of the feature space (since the size of the grid is exponential in dimension). Such a trade-off is ideal in settings with a large number of data points but low dimensionality, which is often the case in computer vision applications. With the growing size of modern datasets and increasing resolution of data collected by sensors and cameras, it is becoming ever more urgent to have fast versions of classical techniques.

Our contributions are as follows:
\begin{itemize}
    \item We propose MeanShift++, a new mode-seeking procedure based on MeanShift that runs in $O(n\cdot 3^d)$ per iteration vs $O(n^2\cdot d)$ for MeanShift. MeanShift++ has no additional hyperparameters over MeanShift.
    \item We show that MeanShift++'s grid-based approximation attains near minimax optimal statistical consistency guarantees at approximating the true density.
    \item An extensive empirical analysis shows that MeanShift++ performs at least as well as MeanShift for clustering while being significantly faster.
    \item Image segmentation results show that MeanShift++ delivers almost identical segmentations as MeanShift while being as much as 10,000x faster. 
    \item Image segmentation experiments on the Berkeley Segmentation Dataset Benchmark (BSDS500) found that MeanShift++ performed on par or better than baselines despite being faster than most (and faster than MeanShift by 1,000x).
    \item We present a new object tracking algorithm based on MeanShift++ that can adapt to gradual color distributions and scene changes--something most MeanShift-based approaches cannot do due to the computational cost.
\end{itemize}

\section{Related Works}

Since MeanShift is a very popular procedure, there have been a number of approaches to speed up the algorithm and other mode-seeking based clustering algorithms in general. 

Yang et al. (2003) \cite{yang2003improved} propose a speedup of MeanShift by using a fast gauss transform to efficiently compute the kernel density estimator, reducing the computational complexity down to linear per iteration; however, they found the fast gauss transform to be impractical for any dimension higher than $3$. Yang et al. (2005) \cite{yang2005efficient} then applies this technique to a modified similarity function specifically for color histograms and show its effectiveness on frame-tracking in image sequences. Elgamma (2003) \cite{elgammal2003efficient} also leverage fast gauss transform for color modeling and tracking.

Then there are other computer vision application specific speedup methods for MeanShift. Yin et al. (2011) \cite{yin2011fast} leverage frame-differences to speed up MeanShift in the specific application of target tracking. Carreira-Perpinan (2006) \cite{carreira2006acceleration} shows that a spatial discretization strategy can accelerate Gaussian MeanShift image segmentation by one to two orders of magnitude while attaining almost the same segmentation. Carreira-Perpinan \cite{carreira2006fast} also provides similar results for general Gaussian MeanShift; however, in this paper, we show that our method can achieve a far better improvement compared to MeanShift.

Another set of approaches leverage space-partitioning data structures in order to speed up the density estimation calculations. Wang et al. (2007) \cite{wang2007fast} propose using a dual-tree to obtain a faster approximation of MeanShift with provable accuracy guarantees. Xiao et al. (2010) \cite{xiao2010efficient} propose a heuristic to make the computations more efficient by approximating MeanShift using a greatly reduced feature space via applying an adaptive Gaussian KD-Tree.

Freedman et al. (2009)  \cite{freedman2009fast} propose speeding up MeanShift by randomly sampling data points when computing the kernel density estimates. This approach only reduces the runtime by small orders and does not address the underlying quadratic runtime issue unless only a small number of samples are used, but this leads to high error in the density estimates. Our method is both linear runtime and utilizes all of the data points to construct an optimal density estimator.

Vedaldi and Soatto (2008) \cite{vedaldi2008quick} propose a procedure called QuickShift, which is modification of MeanShift in which the trajectories of the points are restricted to the original examples. The same procedure was proposed later by Rodriguez and Laio (2014) \cite{rodriguez2014clustering}. The procedure comes with theoretical guarantees \cite{jiang2017consistency}, and GPU-based speedups have been proposed for the algorithm \cite{fulkerson2010really}. We will show later in the experimental results that QuickShift is indeed much faster than MeanShift, but MeanShift++ is still orders of magnitude faster than QuickShift.
\algdef{SE}[DOWHILE]{Do}{doWhile}{\algorithmicdo}[1]{\algorithmicwhile\ #1}%

\section{Algorithm}

We first introduce MeanShift in Algorithm~\ref{alg:meanshift} to cluster data points $X_{[n]} := \{x_1,..,x_n\}$. The most popular version uses a unit flat kernel (i.e. $K(x) := 1[\lVert x\rVert \le 1]$ \cite{cheng1995mean}), but other kernels can be used as well including the Gaussian kernel \cite{carreira2007gaussian}. At each iteration, it moves points to its kernel-weighted mean w.r.t. the last iteration's points until convergence. This computation costs $O(n^2\cdot d)$ time per iteration, even if space-partitioning data structures are used to speed up the search to find the $h$-radius neighborhood in the flat kernel case \cite{xiao2010efficient}.

\begin{algorithm}[H]
\caption{MeanShift \cite{cheng1995mean,comaniciu1999mean,fukunaga1975estimation}}
\label{alg:meanshift}
\begin{algorithmic}[H]
  \State Inputs: bandwidth $h$, tolerance $\eta$, kernel $K$, $X_{[n]}$.
  \State Initialize $y_{0, i} := x_i$ for $i \in [n]$, $t = 1$.
   \Do
  \State For $i \in [n]$: \begin{align*}
      y_{t,i} \leftarrow \frac{\sum_{j \in [n]} K\left(\frac{\lVert y_{t-1, i} - y_{t-1, j} \rVert}{h}\right) y_{t-1, j} }{\sum_{j \in [n]} K\left(\frac{\lVert y_{t-1, i} - y_{t-1, j} \rVert}{h}\right) }.
  \end{align*} 
  \State $t \leftarrow t + 1$.
   \doWhile{$\sum_{i=1}^n \lVert y_{t, i} - y_{t-1, i} \rVert \ge \eta$.}
   \State \Return $\{y_{t,1},...,y_{t, n}\}$.
\end{algorithmic}
\end{algorithm}

We now introduce MeanShift++ (Algorithm~\ref{alg:meanshiftpp}). In each iteration, it first performs a preprocessing step of parititioning the input data points into appropriate grid cells, which are hypercubes of side length $h$. A graphical comparison with MeanShift is shown in Figure~\ref{fig:example}. In practice, we keep two hash tables $\mathcal{C}$ (to store the count in each cell) and $\mathcal{S}$ (to store the sum in each cell). $\mathcal{C}$ and $\mathcal{S}$ are defined as initially empty mappings from lattice points in $D$-dimension (corresponding to grid cells) to non-negative integers and $\mathbb{R}^D$ respectively. We show how they are updated via a single pass through the dataset. 

To compute which grid cell an example belongs to, we first divide each entry by $h$ and then take the element-wise floor function, which gives a $d$-dimensional integer vector index of the grids. Then, we can simply use the preprocessed information in the point's grid cell and neighboring grid cells to compute the shifted point. As a result, each point is moved to the average of all the points within its cell and neighboring cells. This grid-based approach is a much faster approximation of the density at low dimensions. The runtime of MeanShift++ grows linearly with $n$, while MeanShift is quadratic (Figure~\ref{fig:gaussians}).

\begin{algorithm}[H]
\caption{MeanShift++}
\label{alg:meanshiftpp}
\begin{algorithmic}[H]
  \State Inputs: bandwidth $h$, tolerance $\eta$, $X_{[n]}$.
  \State Initialize $y_{0, i} := x_i$ for $i \in [n]$, $t = 1$.
   \Do
  \State Initialize empty hash tables $\mathcal{C}: \mathcal{Z}^d \rightarrow \mathcal{Z}_{\ge 0}$ (stores cell count), $\mathcal{S}: \mathcal{Z}^d \rightarrow \mathbb{R}^d$ (stores cell sum).
  \State $\mathcal{C}(\lfloor y_{t-1, i} / h \rfloor) \leftarrow \mathcal{C}(\lfloor y_{t-1, i} / h \rfloor) + 1$ for $i \in [n]$.
  \State $\mathcal{S}(\lfloor y_{t-1, i} / h \rfloor) \leftarrow \mathcal{S}(\lfloor y_{t-1, i} / h \rfloor) + y_{t-1, i}$ for $i \in [n]$.
  \State Next, for all $i \in [n]$: \begin{align*}
      y_{t,i} \leftarrow \frac{\sum_{v \in \{-1, 0, 1\}^d} \mathcal{S}(\lfloor y_{t-1, i} / h \rfloor + v)  }{\sum_{v \in \{-1, 0, 1\}^d} \mathcal{C}(\lfloor y_{t-1, i} / h \rfloor + v) }.
  \end{align*} 
  \State $t \leftarrow t + 1$.
   \doWhile{$\sum_{i=1}^n \lVert y_{t, i} - y_{t-1, i} \rVert \ge \eta$.}
   \State \Return $\{y_{t,1},...,y_{t, n}\}$.
\end{algorithmic}
\end{algorithm}

\begin{figure}
\begin{center}
\includegraphics[width=\linewidth]{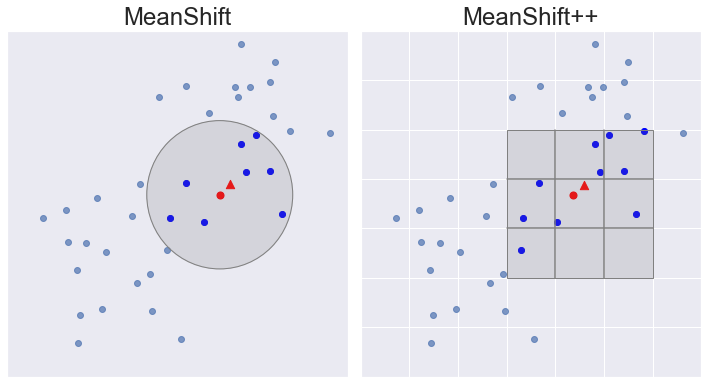}
\end{center}
   \caption{\label{fig:example}\textit{2D example illustrating the difference between MeanShift and MeanShift++.} The red circle is the point we want to shift to the mean of its neighbors. It takes $O(n)$ for MeanShift to find the neighbors of a single point versus $O(3^d)$ for MeanShift++ using grid cells. The location of the new point is indicated by the red triangle.}
\end{figure}
\begin{figure}
\begin{center}
\includegraphics[width=\linewidth]{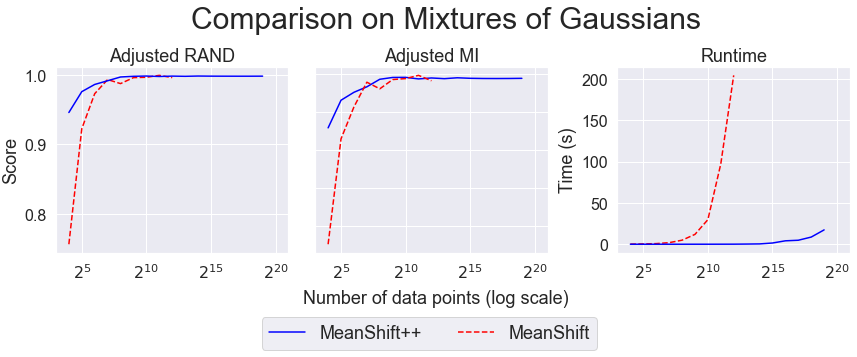}
\end{center}
   \caption{\label{fig:gaussians}\textit{Comparison of MeanShift++ and MeanShift on 2D mixtures of Gaussians.} The runtime of MeanShift is quadratic in the number of data points and quickly becomes infeasible to run, whereas the runtime of MeanShift++ grows linearly with $n$ (shown here in log space). The clustering performances are comparable.}
\end{figure}
\section{Theory}
In this section, we give guarantees on our grid-based approach. Suppose there is some underlying distribution $\mathcal{P}$ with corresponding density function $p : \mathbb{R}^d \rightarrow \mathbb{R}_{\ge 0}$ from which our data points $X_{[n]} = \{x_1,...,x_n\}$ are drawn i.i.d. We show guarantees on the density estimator based on the grid cell counts.

We need the following regularity assumptions on the density function. The first ensures that the density function has compact support with smooth boundaries and is lower bounded by some positive quantity, and the other ensures that the density function has smoothness. These are standard assumptions in analyses on density estimation e.g. \cite{gine2002rates,jiang2017uniform,chen2017tutorial,singh2009adaptive}.
\begin{assumption}\label{assumption1}
$p$ has compact support $\mathcal{X} \in \mathbb{R}^d$ and there exists $\lambda_0, r_0, C_0 > 0$ such that $p(x) \ge \lambda_0$ for all $x \in \mathcal{X}$ and $\text{Vol}(B(x, r) \cap \mathcal{X}) \ge C_0 \cdot \text{Vol}(B(x, r))$ for all $x \in \mathcal{X}$ and $0 < r \le r_0$, where $B(x, r) := \{x' \in \mathbb{R}^d: |x-x'| \le r\}$.
\end{assumption}
\begin{assumption}\label{assumption2}
$p$ is $\alpha$-Hölder continuous for some $0 < \alpha \le 1$: i.e. there exists $C_\alpha > 0$ such that $|p(x) - p(x')| \le C_\alpha \cdot |x - x'|^\alpha$ for all $x, x' \in \mathbb{R}^d$.
\end{assumption}

We now give the result, which says that for $h$ sufficiently small depending on $p$ (if $h$ is too large, then the grid is too coarse to learn a statistically consistent density estimator), and $n$ sufficiently large, there will be a high probability finite-sample uniform bound on the difference between the density estimator and the true density. The proof can be found in the Appendix.
\begin{theorem}\label{theorem}
Suppose Assumption~\ref{assumption1} and~\ref{assumption2} hold. Then there exists constants $C, C_{1} > 0$ depending on $p$ such that the following holds.
Let $0 < \delta < 1$, $0 < h < \text{min}\{\left(\frac{\lambda_0}{2\cdot C_\alpha}\right)^{1/\alpha}, r_0\}$, $nh^d \ge C_1$. Let $\mathcal{G}_h$ be a partitioning of $\mathbb{R}^d$ into grid cells of edge-length $h$ and for $x \in \mathbb{R}^d$. Let $G(x)$ denote the cell in $\mathcal{G}_h$ that $x$ belongs to.  Then, define the corresponding density estimator $\widehat{p}_h$ as:
\begin{align*}
    \widehat{p}_h(x) := \frac{|X_{[n]} \cap G(x)|}{n\cdot h^d}.
\end{align*}
Then, with probability at least $1 - \delta$:
\begin{align*}
    \sup_{x \in \mathbb{R}^d} |\widehat{p}_h(x)  - p(x)| \le C\cdot \left( h^\alpha + \frac{\sqrt{\log(1/(h\delta)}}{\sqrt{n\cdot h^d}} \right).
\end{align*}
\end{theorem}

\begin{remark}
In the above result, choosing $h \approx n^{-1/(2\alpha+d)}$ optimizes the convergence rate to $\tilde{O}(n^{-\alpha/(2\alpha+d)})$, which is the minimax optimal convergence up to logarithmic factors for the density estimation problem as established by Tsybakov \cite{tsybakov1997nonparametric,tsybakov2008introduction}.
\end{remark}
In other words, the grid-based approach statistically performs at least as well as any estimator of the density function, including the density estimator used by MeanShift. It is worth noting that while our results only provide results for the density estimation portion of MeanShift++ (i.e. the grid-cell binning technique), we prove the near-minimax optimality of this estimation. This implies that the information contained in the density estimation portion serves as an approximately sufficient statistic for the rest of the procedure, which behaves similarly to MeanShift, which operates on another, also nearly-optimal density estimator. Thus, existing analyses of MeanShift e.g. \cite{arias2016estimation,chen2015convergence,xiang2005convergence,li2007note,ghassabeh2015sufficient,ghassabeh2013convergence,subbarao2009nonlinear} can be adapted here; however, it is known that MeanShift is very difficult to analyze \cite{dasgupta2014optimal} and a complete analysis is beyond the scope of this paper.

\begin{table}
\begin{tabular}{ |p{0.1cm}||p{3.5cm}|p{1.5cm}|p{0.3cm}|p{0.3cm}| }
        \hline
        & \textbf{Dataset} & $n$ & $d$ & $c$ \\
        \hline
        a & Phone Accelerometer & 13,062,475 & 3 & 7 \\
        \hline
        b & Phone Gyroscope & 13,932,632 & 3 & 7 \\
        \hline
        c & Watch Accelerometer & 3,540,962 & 3 & 7 \\
        \hline
        d & Watch Gyroscope & 3,205,431 & 3 & 7 \\
        \hline
        e & Still & 949,983 & 3 & 6 \\
        \hline
        f & Skin & 245,057 & 3 & 2 \\
        \hline
        g & Iris & 150 & 4 & 3 \\
        \hline
        h & Lupus & 87 & 3 & 2 \\
        \hline
        i & Confidence & 72 & 3 & 2 \\
        \hline
        j & Geyser & 22 & 2 & 2 \\
        \hline
        k & Balance Scale & 625 & 4 & 3 \\
        \hline
        l & Vinnie & 380 & 2 & 2 \\
        \hline
        m & Sleep Data & 1,024 & 2 & 2 \\
        \hline
        n & Transplant & 131 & 3 & 2 \\
        \hline
        o & Slope & 44 & 3 & 2 \\
        \hline
        p & PRNN & 250 & 2 & 2 \\
        \hline
        q & Wall Robot & 5,456 & 4 & 4 \\
        \hline
        r & User Knowledge & 403 & 5 & 5 \\
        \hline
    \end{tabular}
    \vspace{0.3cm}
    \caption{\label{fig:datasetsummary_small}\textit{Summary of datasets used.} Includes dataset size ($n$), number of features ($d$), and number of clusters ($c$).}
\end{table}

\begin{table}
\footnotesize
\begin{tabular}{ |p{2.2cm}||p{0.95cm}|p{1.0cm}||p{0.95cm}|p{1.0cm}| }
        \cline{2-5}
        \multicolumn{1}{c}{} & \multicolumn{2}{|c||}{ARI} & \multicolumn{2}{c|}{AMI}  \\
        \cline{2-5}
        \multicolumn{1}{c}{} & \multicolumn{1}{|c|}{MS++} & \multicolumn{1}{c||}{MS} & \multicolumn{1}{c|}{MS++} & \multicolumn{1}{c|}{MS} \\
        \hline
        a) Phone & \textbf{0.0897} & \textit{DNF} & \textbf{0.1959} & \textit{DNF} \\
        Accelerometer & 29m 59s & $>$24h & 49m 19s & $>$24h \\
        \hline
        b) Phone & \textbf{0.2354} & \textit{DNF} & \textbf{0.1835} & \textit{DNF} \\
        Gyroscope & 1h 35m & $>$24h & 32m 32s & $>$24h \\
        \hline
        c) Watch & \textbf{0.0913} & \textit{DNF} & \textbf{0.2309} & \textit{DNF} \\
        Accelerometer & 17m 52s & $>$24h & 43m 32s & $>$24h \\
        \hline
        d) Watch & \textbf{0.1595} & \textit{DNF} & \textbf{0.1336} & \textit{DNF} \\
        Gyroscope & 24m 45s & $>$24h & 9m 3s & $>$24h \\
        \hline
        e) Still & \textbf{0.7900} & \textit{DNF} & \textbf{0.8551} & \textit{DNF} \\
        & 13.12s & $>$24h & 8.58s & $>$24h \\
        \hline
        f) Skin & \textbf{0.3270} & 0.3255 & \textbf{0.4240} &0.3975 \\
        & 16.44s & 3h 41m & 13.07s & 3h 41m \\
        \hline
        g) Iris & 0.5681 & \textbf{0.6832} & \textbf{0.7316} & 0.6970 \\
        & $<$0.01s & 6.35s & $<$0.01s & 2.36s \\
        \hline
        h) Lupus & \textbf{0.1827} & 0.1399 & \textbf{0.2134} & 0.2042 \\
        & $<$0.01s & 3.82s & $<$0.01s & 3.82s \\
        \hline
        i) Confidence & \textbf{0.2080} & 0.2059 & \textbf{0.2455} & 0.2215 \\
        & 0.02s & 0.70s & $<$0.01s & 0.98s \\
        \hline
        j) Geyser & \textbf{0.1229} & 0.0886 & \textbf{0.2409} & 0.2198 \\
        & $<$0.01s & 2.88s & $<$0.01s & 2.88s \\
        \hline
        k) Balance Scale & \textbf{0.0883} & 0.0836 & \textbf{0.2268} & 0.2166 \\
        & 0.09s & 16.02s & 0.09s & 16.02s \\
        \hline
        l) Vinnie & \textbf{0.4594} & 0.4383 & 0.3666 & \textbf{0.3671} \\
        & 0.01s & 16.85s & 0.01s & 16.85s \\
        \hline
        m) Sleep Data & 0.1181 & \textbf{0.1242} & \textbf{0.1028} & 0.0998 \\
        & 0.02s & 45.25s & 0.02s & 45.25s \\
        \hline
        n) Transplant & \textbf{0.7687} & 0.6328 & \textbf{0.7175} & 0.7018 \\
        & $<$0.01s & 4.22s & $<$0.01s & 4.22s \\
        \hline
        o) Slope & \textbf{0.2777} & 0.2715 & \textbf{0.3877} & 0.3630 \\
        & $<$0.01s & 0.43s & $<$0.01s & 0.43s \\
        \hline
        p) PRNN & \textbf{0.2093} & 0.1872 & \textbf{0.2912} & 0.2590 \\
        & 0.02s & 10.72s & $<$0.01s & 10.72s \\
        \hline
        q) Wall Robot & \textbf{0.1788} & 0.1706 & 0.3239 & \textbf{0.3246} \\
        & 0.69s & 4m37s & 0.88s & 2m30s \\
        \hline
        r) User Knowledge & \textbf{0.3398} & 0.2140 & \textbf{0.4086} & 0.3278 \\
        & 0.06s & 7.62s & 0.06s & 7.62s \\
        \hline
    \end{tabular}
    \vspace{0.3cm}
    \caption{\label{fig:datasetsummary_large}\textit{Summary of clustering performances.} MeanShift++'s and MeanShift's best results for 19 real-world datasets after tuning bandwidth. Datasets from the UCI Machine Learning Repository \cite{Dua:2019} and OpenML \cite{OpenML2013}. In cases where the original target variable is continuous, binarized versions of the datasets were used. These experiments were run on a local machine with a 1.2 GHz Intel Core M processor and 8 GB memory. MeanShift did not finish (DNF) within 24 hours for the top five largest datasets. ARI, AMI, and runtime are reported for each run, and the highest score obtained for that metric and dataset is bolded.}
\end{table}

\begin{figure*}
\begin{center}
\includegraphics[width=0.9\linewidth]{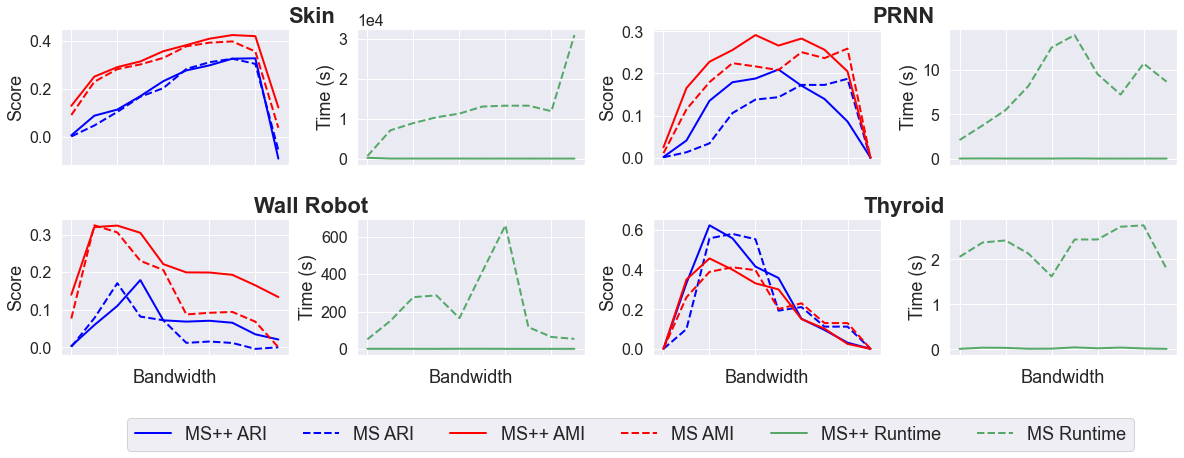}
\end{center}
   \caption{\label{fig:experiments}\textit{Comparison of MeanShift++ and MeanShift on four real-world datasets across a wide range of hyperparameter.} Datasets are shown here to illustrate how both algorithms were tuned over an appropriate range of bandwidth. Adjusted RAND index (ARI), adjusted mutual information score (AMI), and runtime are reported for each run. MeanShift++ consistently performs as well or better than MeanShift despite being up to 1000x faster. Additional experiments are shown in the Appendix.}
\end{figure*}

\begin{figure*}
\begin{center}
\includegraphics[width=\linewidth]{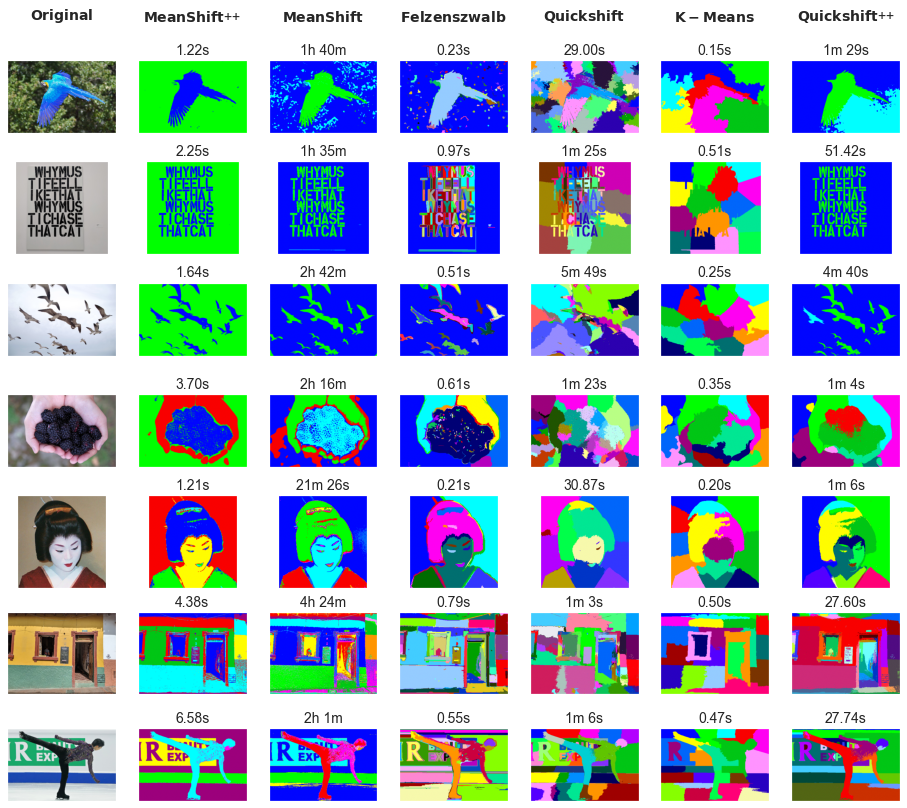}
\end{center}
   \caption{\label{fig:image_segmentation}\textit{Comparison of six image segmentation algorithms.} We show the results of MeanShift++, MeanShift, Quickshift++, and three other popular image segmentation algorithms from the Scikit-Image library\cite{van2014scikit}. MeanShift returns qualitatively good results on image segmentation but takes very long to run. MeanShift++ returns segmentations that are the most similar to MeanShift with an up to 10,000x speedup. We expect the speedup to be far greater for high resolution images--the images shown here are low resolution (under 200k pixels).}
\end{figure*}

\begin{figure*}
\begin{center}
\includegraphics[width=1.0\linewidth]{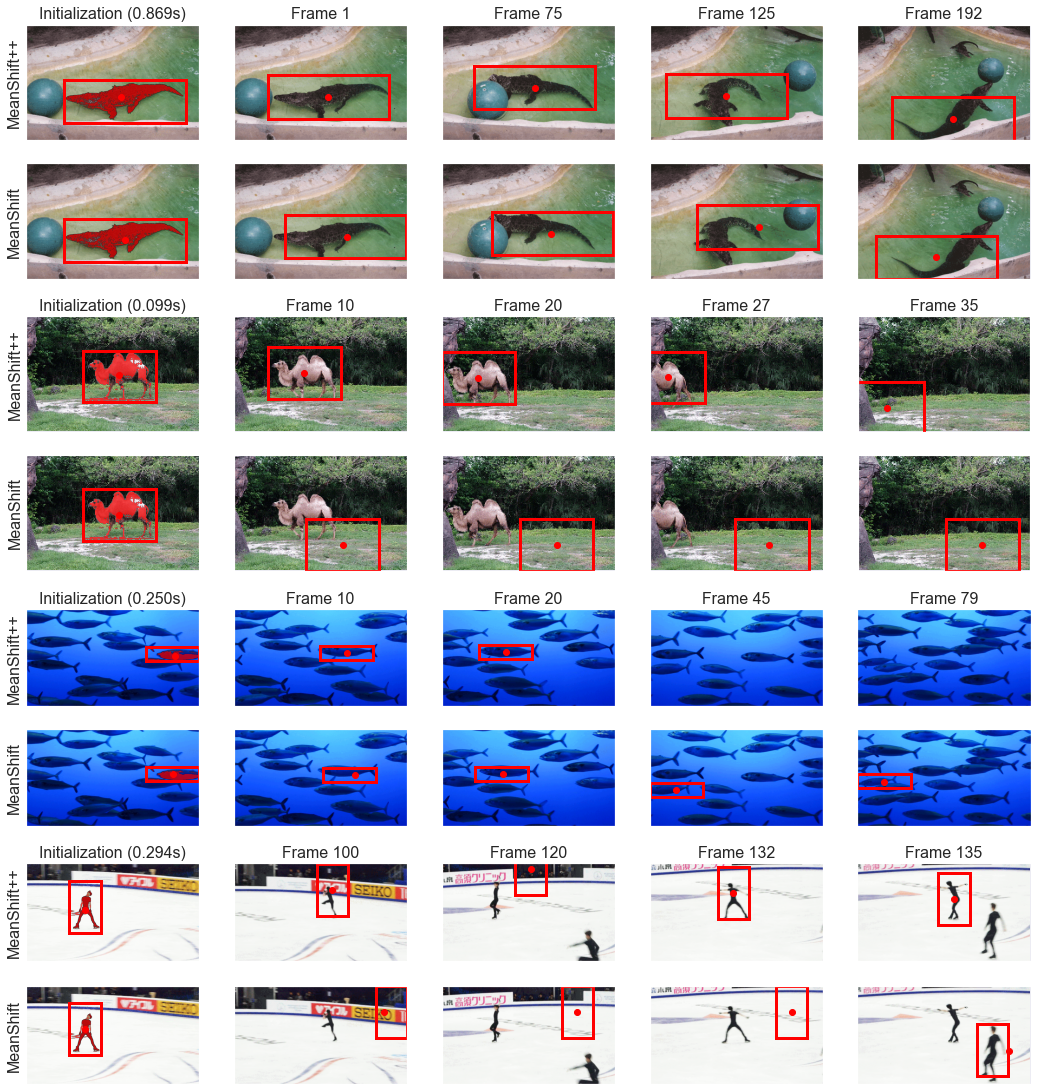}
\end{center}
   \caption{\label{fig:object_tracking}\textit{Comparison of MeanShift++ and MeanShift on object tracking.} Unlike MeanShift++, MeanShift is too slow to generate masks for real-time object tracking. In practice, the user manually provides a color range that they want to track, which is often incomplete, inaccurate, or biased. Here, we initialize both MeanShift++ and MeanShift with a mask from clustering results generated by MeanShift++ to save time. For MeanShift, we use OpenCV's \cite{bradski2008learning} implementation of color histograms to track the object in question. For MeanShift++, we naturally use the grid cells that are returned from the clustering step. We find that MeanShift is more likely to get distracted by backgrounds, foregrounds, and other objects in the scene. {\bf First scene}: MeanShift returns less accurate object centers and search windows. {\bf Second scene}: MeanShift fails to find the object altogether due to an abundance of similar colors in the frame that cannot be decoupled from the object of interest. {\bf Third scene}: MeanShift starts tracking similar objects nearby when the original objective moves out of frame. In contrast, MeanShift++ stops tracking when it finds the center of mass in the search window disappear. {\bf Fourth scene}: MeanShift loses the skater faster than MeanShift++ and fails to find him again (instead it starts to track another skater altogether).}
\end{figure*}

\begin{figure}
\begin{center}
\includegraphics[width=\linewidth]{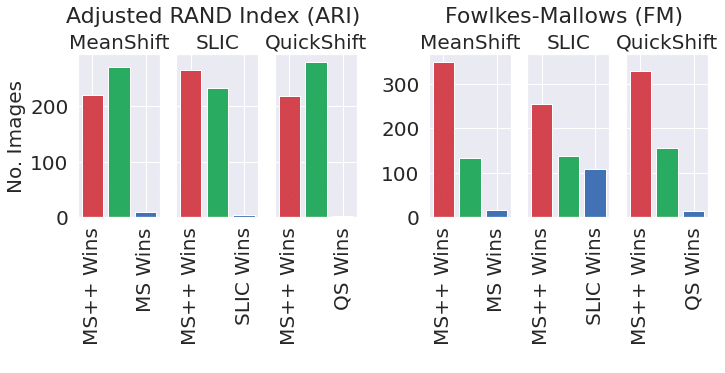}
\begin{tabular}{ |p{2.2cm}||p{3.8cm}| }
    \hline
    & Average Runtime ($\mu s$) \\
    \hline
    MeanShift++ & 2,675,100 \\
    \hline
    MeanShift & 1,765,462,893 \\
    \hline
    SLIC & 56,266 \\
    \hline
    Quickshift & 31,717,386 \\
    \hline
\end{tabular}
\end{center}
   \caption{\label{fig:bsds500}\textit{Comparison of MeanShift++, MeanShift, SLIC, and QuickShift on BSDS500 using the ARI and FM clustering metrics.} Performance metrics are averaged over 500 images. For each baseline, we plot MS++ wins, baseline wins, and ties (where the two algorithms score within 1\% of each other). MS++ performs on par or better compared to baseline algorithms. In order for MeanShift++ to finish running, we sample the images down by an order of 2. MeanShift++ is still around 1000x faster than MeanShift.}
\label{fig:short}
\end{figure}
\section{Experiments}
We compare Meanshift++ against MeanShift on various clustering tasks in Table~\ref{fig:datasetsummary_small}. These comparisons are made using the Scikit-Learn \cite{pedregosa2011scikit} implementation of MeanShift and our own implementation of MeanShift++ in Cython.

To measure the quality of a clustering result, we use the Adjusted Rand Index (ARI) \cite{hubert1985comparing} and the Adjusted Mutual Information (AMI) \cite{vinh2010information} scores, which compare the clustering with the partitioning induced by the labels of the data points, a popular way of comparing clustering performance \cite{jang2019dbscan++}. The benchmark datasets we use are labeled datasets, and we only cluster the features. 

As stated earlier, MeanShift++ is linear with respect to the number of data points and exponential in dimension. We thus show results on low-dimensional datasets. In Figure~\ref{fig:experiments}, we ran both algorithms on 19 benchmark datasets with $5$ or fewer dimensions, ranging from less than $100$ data points to millions of data points.

For the top five largest datasets, MeanShift failed to return a result for any setting of bandwidth despite running for more than 24 hours. MeanShift++ consistently outperformed MeanShift in both clustering quality and runtime for the rest of the datasets, as shown in Table~\ref{fig:datasetsummary_large}. We saw a significant speed reduction of over 100x on both small and large datasets, showing that MeanShift++ does not have significantly more overhead costs than MeanShift either. 

We also show in Figure~\ref{fig:experiments} the effect the bandwidth setting has on clustering performance and runtime for a few of the datasets to provide further insight into the stability of the procedures under the bandwidth hyperparameter.

We note that MeanShift++ outperforms MeanShift on many datasets, possibly due to a regularizing effect: by partitioning the space into grids and assigning every point in the same cell the same value instead of a unique value for each point, the gradient-ascent shifting step is more stable than in MeanShift. This regularization effect, combined with the option to tune the cell-length which essentially controls the amount of regularization, allows MeanShift++ to outperform Meanshift in some cases. 

However, these results are unlikely to generalize to higher dimensions. It is known that density-based procedures perform poorly in high dimensions due to the curse of dimensionality. Our theoretical results also show that rates become exponentially worse in higher dimension. 
\section{Image Segmentation}

We compare MeanShift++ to a number of baselines for unsupervised image segmentation in Figure~\ref{fig:image_segmentation}. We include Felzenszwalb \cite{felzenszwalb2004efficient}, QuickShift \cite{vedaldi2008quick}, and $k$-means, three popular image segmentation procedures from the Python Scikit-Image library \cite{van2014scikit}, as well as Quickshift++ \cite{jiang2018quickshift++}, a recent algorithm shown to be an improvement over Quickshift on image segmentation. We also include MeanShift, which often produces qualitatively better clusters than the other baselines, but runs for so much longer that it is impractical for high-resolution image segmentation. 

For image segmentation, we run each algorithm on a preprocessed image with each pixel represented in a 3D RGB color channel space, with the exception of Quickshift++, which takes $(r, g, b, x, y)$ color and spatial coordinates. MeanShift was run with both $(r, g, b)$ (shown in Figure~\ref{fig:image_segmentation}) and $(r, g, b, x, y)$ inputs, but we did not see a difference in segmentation quality or runtime. For each algorithm, the returned clusters are taken as the segments.

Our image segmentation experiments in Figure~\ref{fig:image_segmentation} show that MeanShift++ is able to produce segmentations that are nearly identical to that of MeanShift with an up to 10,000x speedup. We capped the sizes of our images at 187,500 pixels to allow MeanShift to finish running, so this speedup would surely be greater on even higher resolution images. 

Multiple attempts have been made to speed up MeanShift for image segmentation at the cost of quality, but MeanShift++ does not seem to trade off segmentation quality despite running in a sub-fraction of the time.

For a more quantitative comparison, we ran experiments using the Berkeley Segmentation Dataset Benchmark (BSDS500) of $500$ images with $6$ human-labeled segmentations each. We ran MeanShift++, MeanShift, SLIC, and QuickShift on each image and used the adjusted RAND index \cite{hubert1985comparing} (ARI) and Fowlkes-Mallows \cite{fowlkes1983method} (FM) scores to compare the clusters. Scores were averaged over the $6$ ground truth segmentations. We found that MeanShift++ performed on par or better than baselines despite being faster than MeanShift by 1,000x on average (Figure \ref{fig:bsds500}). \\
\section{Object Tracking}

The mode-seeking behavior of MeanShift makes it a good candidate for visual tracking. A basic implementation would take in a mask and/or search window of an object and build a histogram of the colors found in that object in RGB or HSV space. Afterwards, the histogram is normalized by the colors found in surrounding, non-target points into a probability distribution. At each step, the tracking algorithm would backproject each point in the window into a probability that it is part of the original object. The center of the window moves in $(x, y)$ space to the mode of the distribution within the window until convergence in the same way that MeanShift iteratively moves each point to the mean of its neighbors. Every frame thereafter would be initialized with the final window of the previous frame.

\begin{algorithm}
\caption{MeanShift++ for Tracking}
\label{alg:meanshiftpp_tracking}
\begin{algorithmic}[H]
  \State {\bf Inputs}: bandwidth $h$, tolerance $\eta$, initial window $W_0$, sequence of frames $X_0,X_1,...,X_T$.
  \State Define: $W \cap X$ as pixels in window $W$ for frame $X$.
  \State Run MeanShift++ on the pixels (in color space) in $W_0 \cap X_0$ and manually select the cluster(s) desired to track. Let the union of selected cluster(s) be $C$.
  \State $B \leftarrow \{\lfloor c / h \rfloor  : c \in C\}$.
  \For{$i=1,2,...,T$}
    \State Initialize $W_i \leftarrow W_{i-1}$.
    \Do
    \State $R_i := \{ x \in W_i \cap X_i : \lfloor x / h \rfloor \in B \}$.
    \State Move $W_i$ so that it's centered at the average $(x,y)$-position of points in $R_i$.
    \doWhile{$W_i$'s center converges with tolerance $\eta$.}
    \State Optionally update $B \leftarrow \{\lfloor x / h \rfloor  : x \in W_i \cap X_i\} \cap N(B)$, where $N(B)$ are all cells that are in $B$ or adjacent to one in $B$.
    \State {\bf emit} $W_i$ for frame $X_i$.
  \EndFor
\end{algorithmic}
\end{algorithm}

MeanShift++ can be used for object tracking in a similar, albeit more principled way (Algorithm~\ref{alg:meanshiftpp_tracking}). Instead of color histograms, which need to be computed and require extra hyperparameters to determine the size of bins, the grid cells generated by MeanShift++ during clustering are already precomputed and suitable for tracking: we can quickly compute which points fall into any bin belonging to the target cluster(s).

MeanShift is also too slow to generate masks needed for real-time tracking. It often requires the user to provide a precomputed mask or color range. Relying on user input is imperfect and subject to biases. MeanShift++ is fast enough to generate masks through real-time clustering. 

Off-the-shelf versions of MeanShift tracking rely on a histogram calculated from the original frame throughout the whole scene. This does not work well if the illumination in the scene changes, since the algorithm cannot make fast updates to the histogram \cite{freedman2005illumination,whoang2012object,phadke2013illumination,phadke2017mean}. MeanShift++ can adapt to changing color distributions by finding and adding neighboring grids of points to the histogram in linear time, making it more robust to gradual changes in lighting, color, and other artifacts of the data.

CamShift \cite{bradski1998computer,allen2004object} improves MeanShift tracking by adjusting window sizes as objects move closer or farther away and updating color histograms based on lighting changes, among other things. Future work may involve adapting these ideas to a MeanShift++-based tracking algorithm.

In Figure~\ref{fig:object_tracking}, we show the performance of MeanShift++ and MeanShift (from the Python OpenCV library \cite{bradski2008learning}) on object tracking in various scenes. In practice, we found that MeanShift tends to be easily misled by the surroundings, particularly when there are areas of similar color. Finally, MeanShift++ usually finds better windows than MeanShift.

\section{Conclusion}

We provided MeanShift++, a simple and principled approach to speed up the MeanShift algorithm in low-dimensional settings. We applied it to clustering, image segmentation, and object tracking, and show that MeanShift++ is competitive with MeanShift in low dimensions while being as much as 10,000x faster. This dramatic speedup makes MeanShift++ practical for modern computer vision applications. 

\clearpage
{\small
\bibliographystyle{ieee_fullname}
\bibliography{paper}
}
{
\appendix
\clearpage
\onecolumn

\section{Experiments}
\begin{figure}[H]
\begin{center}
\includegraphics[width=0.95\linewidth]{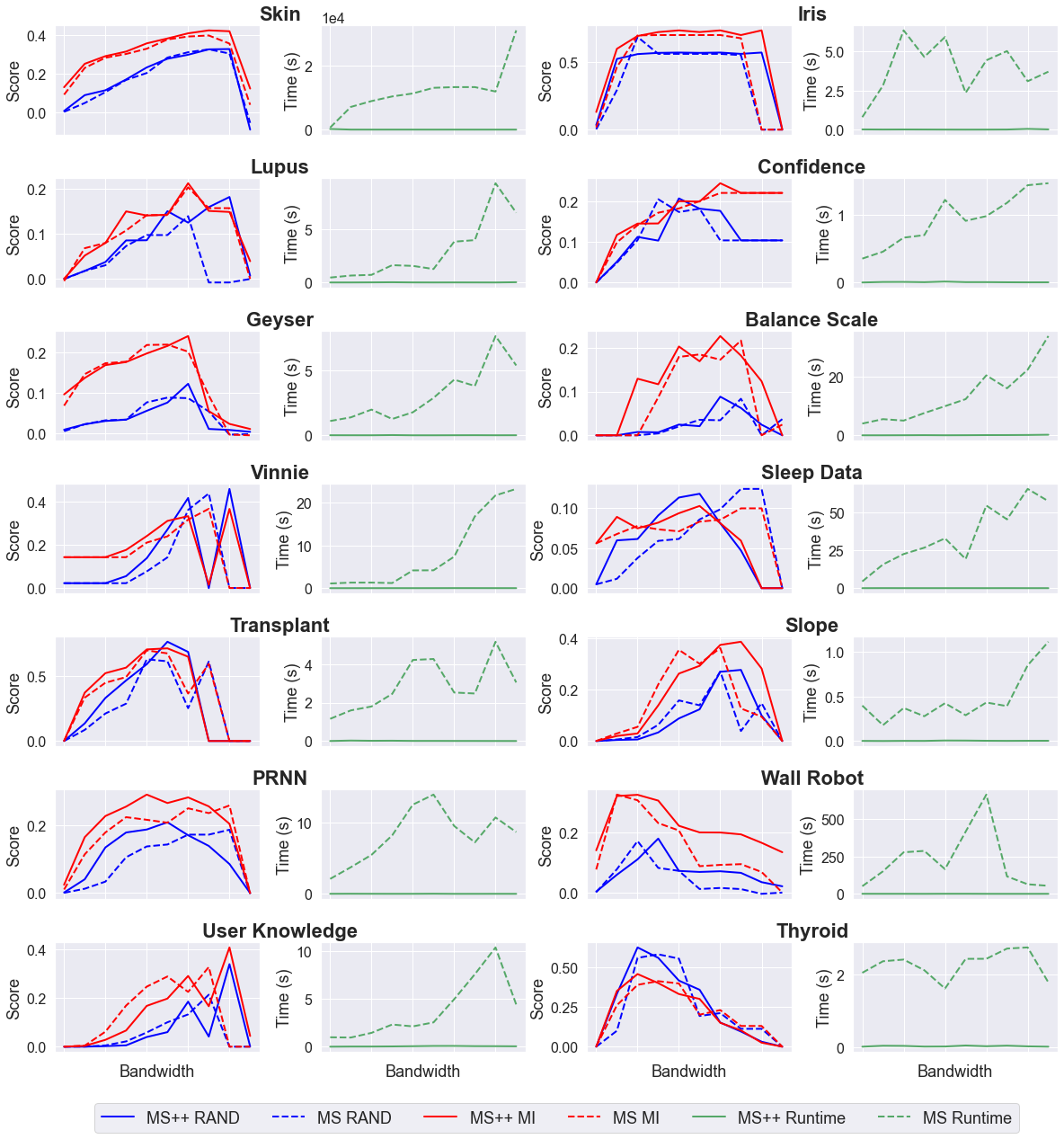}
\end{center}
   \caption{\label{fig:additional_experiments}\textit{Additional experiments.} MeanShift++ and MeanShift performances on 10 more real-world datasets, for a total of 14 small datasets. Again, we illustrate how both algorithms were tuned over an appropriate range of bandwidth, and ARI, AMI, and runtime are reported for each run. MeanShift++ consistently performs as well or better than Meanshift in a fraction of the time, besting MeanShift in 12 out of these 14 select small datasets in ARI and AMI.}
\end{figure}

\section{Proofs}

\begin{proof}[Proof of Theorem~\ref{theorem}]
We first give an upper and lower bound on $\mathcal{P}(G)$ for $G \in \mathcal{G}_h$, where $\mathcal{P}(G)$ denotes the total probability mass of $G$ w.r.t. $\mathcal{P}$.
Since $\mathcal{X}$ is bounded in $\mathbb{R}^d$, there exists a constant $C_\X$ such that for all $h > 0$, we have $|\mathcal{G}_h| \le C_\X \cdot h^{-d}$.
Next, we have by the assumptions that for any cell $G \in \mathcal{G}_h$, for some $C'$ depending on $p$:
\begin{align*}
    \mathcal{P}(G) &\ge \min_{x \in G\cap \mathcal{X}} p(x) \cdot \text{Vol}(G \cap \mathcal{X})\\
    &\ge C' (\lambda_0 - C_\alpha\cdot h^\alpha)\cdot h^d \ge \frac{1}{2} C' \lambda_0 \cdot h^d.
\end{align*}
Next, since $\mathcal{X}$ is compact, there exists $p_{\text{max}}$ such that $\sup_{x\in\X} p(x) = p_{\text{max}} < \infty$. Thus,
$\mathcal{P}(G)\le  p_{\text{max}}\cdot h^d$.

We now bound $\sup_{x \in \X} |\frac{\mathcal{P}(G(x))}{h^d} - \widehat{p}_h(x)|$. We have that the event a sample drawn according to $x$ lies in $G$ is a Bernoulli random variable of probability $\mathcal{P}(G)$. From the above, we have that this variance is upper and lower bounded by $O(h^{d})$, and let us denote the number of samples in $\X_{[n]}$ that lie in $G$ as $\mathcal{P}_n(G)$. 
Therefore, by Hoeffding-Chernoff inequality (i.e. Theorem 1.3 of \cite{phillips2012chernoff}), we have for some $C''$ depending on $p$ that
\begin{align*}
    \mathbb{P}\left(\left|\cdot \mathcal{P}(G(x)) - \cdot \mathcal{P}_n(G(x))\right| \ge \frac{t}{n}\right) \le \exp\left(-\frac{t^2}{4C''\cdot n\cdot h^d}\right),
\end{align*}
for $n \cdot h^d$ sufficiently large depending on $p$.
Now choosing $t = 2\sqrt{n\cdot h^d}\cdot \sqrt{  C''\cdot\log(C_\mathcal{X}\cdot h^{-d}/\delta)}$, we have
\begin{align*}
    \mathbb{P}\left(\left|\frac{\mathcal{P}(G(x))}{h^d} - \widehat{p}_h(x)\right| \ge \frac{2\sqrt{C''\cdot \log( C_\mathcal{X}\cdot h^{-d}/\delta)}}{\sqrt{n\cdot h^d}}\right) \le \frac{\delta}{|\mathcal{G}_h|}.
\end{align*}
Thus, by union bound, we have the following holds:
\begin{align*}
    \mathbb{P}\left(\sup_{x\in \mathcal{X}}\left|\frac{\mathcal{P}(G(x))}{h^d} - \widehat{p}_h(x)\right| \ge \frac{2\sqrt{C''\cdot\log( C_\mathcal{X}\cdot h^{-d}/\delta)}}{\sqrt{n\cdot h^d}}\right) \le \delta.
\end{align*}
Next, we have by the smoothness assumption that
\begin{align*}
    \sup_{x\in \X} \left|p(x) - \frac{\mathcal{P}(G(x))}{h^d} \right| \le C_\alpha h^\alpha.
\end{align*}
The result follows by triangle inequality.
\end{proof}

}
\end{document}